\documentclass[manuscript, screen]{acmart}
\usepackage{booktabs}

 \usepackage{multirow}

\usepackage{graphicx,color}
\usepackage{array,float}
\usepackage{url}
\usepackage{amstext,amssymb,amsmath}
\usepackage{amsthm}
\usepackage{verbatim}
\usepackage{bm}
\usepackage{paralist}
\usepackage{ulem}
\usepackage{todonotes}
\usepackage{hyperref}
\usepackage[noend]{algorithmic}
\usepackage{algorithm}

\usepackage{breqn}

\usepackage{graphicx}
\usepackage{caption}
\usepackage{subcaption}

\usepackage{pgfplots}

\usepackage{booktabs}
\usepackage{array}
\newcolumntype{L}[1]{>{\raggedright\let\newline\\\arraybackslash\hspace{0pt}}m{#1}}
\newcolumntype{C}[1]{>{\centering\let\newline\\\arraybackslash\hspace{0pt}}m{#1}}
\newcolumntype{R}[1]{>{\raggedleft\let\newline\\\arraybackslash\hspace{0pt}}m{#1}}

\def\header{\vspace{1.5mm} \noindent}

\pgfplotsset{soldot/.style={color=black,only marks,mark=*}} \pgfplotsset{holdot/.style={color=black,fill=white,only marks,mark=*}}

\colorlet{shadecolor}{blue!20}

\newcount\colveccount
\newcommand*\colvec[1]{
        \global\colveccount#1
        \begin{pmatrix}
        \colvecnext
}
\def\colvecnext#1{
        #1
        \global\advance\colveccount-1
        \ifnum\colveccount>0
                \\
                \expandafter\colvecnext
        \else
                \end{pmatrix}
        \fi
}

\newtheorem{lem}{Lemma}[section]

\newtheorem{thm}[lem]{Theorem}

\newcommand{\SB}{\textsf{SB}}
\newcommand{\TB}{\textsf{TB}}

\newcommand{\WT}{\textsf{WT}}
\newcommand{\FL}{\textsf{FL}}

\renewcommand{\paragraph}[1]{\vspace{ authors would like to3pt}\noindent\textbf{#1}}

\newcommand{\A}{\mathcal{A}}

\newcommand{\D}{\mathcal{D}}

\newcommand{\UU}{\mathcal{U}}

\newcommand{\Lf}{\mathcal{L}}

\newcommand{\p}{^\prime}

\newcommand{\mynote}[3]{}

\newcommand{\ignore}[1]{}

\newcommand{\pdf}{\operatorname{pdf}}

\usepackage{enumitem}

\begin{document}

	\title{Differentially Private Regression for Discrete-Time Survival Analysis}

\author{Th\^{o}ng T. Nguy\^{e}n}
\affiliation{
  \institution{Nanyang Technological University}
}
\email{s140046@ntu.edu.sg}

\author{Siu Cheung Hui}
\affiliation{
  \institution{Nanyang Technological University}
}
\email{asschui@ntu.edu.sg}

\renewcommand{\shortauthors}{T.T. Nguy\^{e}n et al.}

\begin{abstract}
In survival analysis, regression models are used to understand the effects of explanatory variables (e.g., age, sex, weight, etc.) to the survival probability. However, for sensitive survival data such as medical data, there are serious concerns about the privacy of individuals in the data set when medical data is used to fit the regression models. The closest work addressing such privacy concerns is the work on Cox regression which linearly projects the original data to a lower dimensional space. However, the weakness of this approach is that there is no formal privacy guarantee for such projection. In this work, we aim to propose solutions for the regression problem in survival analysis with the protection of differential privacy which is a golden standard of privacy protection in data privacy research. To this end, we extend the \emph{Output Perturbation} and \emph{Objective Perturbation} approaches which are originally proposed to protect differential privacy for the Empirical Risk Minimization (ERM) problems. In addition, we also propose a novel sampling approach based on the Markov Chain Monte Carlo (MCMC) method to practically guarantee differential privacy with better accuracy. We show that our proposed approaches achieve good accuracy as compared to the non-private results while guaranteeing differential privacy for individuals in the private data set.
\end{abstract}



\maketitle

\section{Introduction}

Survival analysis studies and models probability of failure of time-related processes (e.g., time to death of HIV patients, time to divorce of married couples, time to graduation of Ph.D. students, etc.). Two important concepts in survival analysis are (1) the hazard rate function $h(t)$ which is the probability of failure (death) at time $t$, and (2) the survival function $S(t)$ which is the probability of survival to time $t$. An example of survival data set is the electronic health records (EHRs) which have been widely used and collected at large scale in modern hospitals \cite{desroches2008electronic,jha2009use,blumenthal2010meaningful}. These health records are very useful for fitting the regression models to assist doctors in the medical decision processes for treatment, diagnosis, etc.
In general, regression models are used to analyze the effects of explanatory variables (e.g., age, sex, weight, etc.) to the survival probability of patients.
However, these models may also have serious problems of breaching patient's privacy as there is no guarantee that these models do not leak any personal information of individual patients in the data set.

In this work, we focus on the privacy problems of regression models used in survival analysis.
We consider the setting in which privacy-preserving algorithms use data in the private data set to fit a survival regression model. The model is then published and available to the public for the benefits of society. Therefore, in this setting, the adversaries are assumed to know the output model, i.e., the parameters of the regression model. The goal is to design algorithms that can fit the survival regression model to the data set with high accuracy while guaranteeing that the adversaries cannot learn much information about the individuals in the data set when knowing the output model.

There are two different kinds of regression models in survival analysis, namely continuous-time models and discrete-time models. For continuous-time models, time is a continuous variable and failure events can happen at any moment. Cox regression is a well-known continuous-time model \cite{cox1992regression, andersen1982cox} which allows estimation without any assumption on the baseline hazard effects. However, we have to assume the proportional hazard property (i.e., a unit increase in an explanatory variable
will cause a multiplicative effect on the hazard rate). For discrete-time models \cite{allison1982discrete, cox1984analysis, muthen2005discrete}, time is discrete and failure events only happen at discrete values of time. Discrete-time regression models are better than Cox regression when dealing with tied events (i.e., events which have the same value of survival time) and unobserved population heterogeneity (i.e., unobserved explanatory variables may cause bias to the estimation). Moreover, it does not need the proportional hazard property assumption as Cox regression does \cite{hess2012duration}.

In this paper, we propose solutions for the problem of guaranteeing discrete-time models not to leak personal information of the patients.
Our proposed approaches guarantee differential privacy protection, which is the state-of-the-art privacy-preserving technique in data privacy research. Informally, a differentially private algorithm guarantees that two neighboring data sets which are different at only one patient's record are guaranteed to produce two outputs whose probability densities are very similar. This prevents an adversary from recognizing a data set from the collection of its neighbors. Therefore, an adversary cannot infer the personal information of a particular patient in the data set even in the case when the adversary knew all the information of all other patients in the data set (if otherwise, then the adversary can easily distinct two neighboring data sets).

In our solutions, we use the maximum likelihood estimation to transform the estimation problem to the optimization problem of choosing parameters to maximize the log-likelihood of the observed data set with respect to the discrete-time model.
Coincidentally, our problem has a similar likelihood form as a logistic regression problem. This allows us to use the Output Perturbation (\textsf{Out-Pert}) and Objective Perturbation (\textsf{Obj-Pert}) proposed by Chaudhuri et al. \cite{chaudhuri2011differentially} for our problem. These methods were originally proposed to protect differential privacy for the Empirical Risk Minimization (ERM) problems which include the logistic regression problem. The \textsf{Out-Pert} approach adds noise to the optimization solution to protect differential privacy. The \textsf{Obj-Pert} approach randomly perturbs the objective function, thereby ensuring the randomness of its optimization solution which can guarantee differential privacy for the solution. However, these approaches cannot be applied directly to our problem due to the difference in the loss function. Especially, this is due to the fact that our loss function is not a logistic loss function but a sum of logistic loss functions as the result of the discrete-time models. Therefore, we propose generalized extensions of the \textsf{Out-Pert} and \textsf{Obj-Pert} approaches to cater for our loss function.

A disadvantage of the above perturbation approaches is that for them to work properly they require a non-negligible regularization term in the objective function which incurs bias to the output model. To tackle this, we propose a sampling approach which protects differential privacy by directly sampling parameters from the objective function without the need of a regularization term to guarantee differential privacy. Similar ideas on sampling the objective functions to provide differential privacy are also proposed in \cite{wang2015privacy, kifer2012private, bassily2014private} for the ERM problems. However, it is required that the loss function has to have a finite maximum value. The previous works guarantee this property by boxing the output parameters in a finite-volume space (e.g., a sphere). This approach does not work well when the optimal parameter has a large magnitude. In this work, to guarantee the finite constraint, we wrap the loss function inside a sanitizer function (i.e., a scaled $tanh$ function) to create a new finite loss function.  We intentionally pick the sanitizer function that can keep the loss function in its original form when the value of the loss function is small. Meanwhile, the sanitizer function deforms the loss function at large values to make the function finite. The advantage of this approach is that the sampled parameter can arbitrary large while the objective function is kept almost the same around the optimal parameter which minimizes the objective function.

In order to sample an output parameter from the posterior distribution,  Bassily et al. \cite{bassily2014private} proposed a polynomial run-time algorithm to sampling the log-concave objective function but their algorithm is still impractical due to the high degree of its polynomial run-time complexity. On the other hand, Wang et al. \cite{wang2015privacy} proposed to use a stochastic gradient Nos{\'e}-Hoover thermostat algorithm \cite{ding2014bayesian} to sample the posterior distribution. In this work, we propose to use Preconditioned Stochastic Gradient Langevin Dynamics (pSGLD) sampling algorithm \cite{li2015preconditioned} to sample the objective function due to its advantages in sampling multi-dimensional parameters with different scales. It is worth to note that even though the sampling approach gives better accuracy (as we will see in Section 6), due to the property of its Markov chain, it cannot sample the objective function exactly. Therefore, the sampling approach does not mathematically guarantee differential privacy but only guarantees it approximately in practice.

In summary, the main contributions of this paper are as follows:
\begin{itemize}
    \item We propose two privacy-preserving approaches, namely the Extended Output Perturbation and Extended Objective Perturbation, for the discrete-time survival regression problem. The proposed approaches guarantee differential privacy for the survival regression models. We formally prove these guarantees based on the definition of differential privacy.
     \item We propose a sampling approach to output a random model from its posterior distribution. The proposed sampling approach is based on pSGLD, which is a particular kind of the Markov Chain Monte Carlo (MCMC) method, to efficiently sample the random output which  guarantees differential privacy approximately in practice.
     \item We show the effectiveness of our proposed approaches on four real survival data sets. In addition, we show that the results obtained from the discrete-time models are very close to the results obtained from Cox regression. We also show experimentally the convergence of our proposed sampling approach.
 \end{itemize}

The rest of the paper is organized as follows: In Section~\ref{sec2}, we review the related work on differential privacy and discrete-time survival analysis. Section~\ref{sec3} presents the regression models used in this work. Sections~\ref{sec4} discusses the proposed approaches of the Extended Output Perturbation and Extended Objective Perturbation along with their privacy guarantees. Section~\ref{sec5} discusses the proposed sampling approach. Section~\ref{sec6} presents the experimental results from real data sets. Finally, we conclude the paper in Section~\ref{sec7}.

\section{Related Work}
\label{sec2}

Even though it is important to protect privacy in medical data, as far as we know the work of Yu et al. \cite{yu2008privacy} is the only work on privacy protection for Cox regression. Their work considers the setting in which Cox regression is executed on a distributed data set over many institutions. They proposed to project patient's data to a lower dimensional space by a linear projection. The projection is satisfied by an optimization constraint to preserve good properties of the original data. However, their work is not based on a formal privacy definition such as differential privacy.
Our work on discrete-time models for survival analysis is the first to propose a solution for the privacy problem of discrete-time survival models and also the first to apply differential privacy to survival analysis.

\subsection{Differential Privacy}

The state-of-the-art technique for the data privacy problem is \emph{differential privacy} \cite{dwork2011firm, dwork2014algorithmic,dwork2009differential}.
Basically, differential privacy is a promise to individuals in the data set that their information will not  influence much on the final published results from the analysis.
Differential privacy is used in many applications such as histogram publication \cite{zhang2014towards, li2010optimizing},
graph analysis \cite{kasiviswanathan2013analyzing,lu2014exponential, borgs2015private}, regression  and classification \cite{chaudhuri2009privacy, kifer2012private,bassily2014private, wang2015privacy}, recommender systems \cite{machanavajjhala2011personalized,mcsherry2009differentially}, etc.
Here, we give a brief overview of differential privacy, interested readers can refer to \cite{dwork2014algorithmic} for a detailed discussion on this subject.

To formalize the definition of differential privacy, we first need to introduce the definition of \emph{two neighboring data sets.}

\begin{definition}[Neighboring data sets] Two data sets $\D$ and $\D\p$ are neighbors (denoted as $d(\D, \D\p) = 1$) if they agree in all except one record.
\end{definition}
 From that, we have a formal definition of differential privacy.
\begin{definition}[Differential privacy]
An algorithm $\A$ is $\epsilon-$differentially private if for any output value $x$ of $\A$ and for any pair of neighboring data sets $\D$ and $\D\p$:
$$
\pdf( \A(\D) = x) \leq \exp(\epsilon) \cdot \pdf(\A(\D\p) = x)
$$ where $\epsilon$ is the privacy budget of the algorithm $\A$.
\end{definition}

\subsection{Discrete-time Survival Analysis}

For discrete-time models, let time be divided into intervals $[a_0, a_1)$, $[a_1, a_2)$, $\dots$, $[a_{q-1}, a_{q}], a_0 =0, a_q = 1, $ where $q$ is the number of discrete times. The discrete time $t$ refers to the interval $[a_{t-1}, a_t)$. A discrete random variable $T$ represents the discrete failure time. $T = t$ denotes the failure within the time interval $t = [a_{t-1}, a_t)$. The characteristic function of $T$ is the discrete hazard function:
$$
h(t) = \Pr(T=t \mid T \geq t),~~~ t = 1, \dots, q
$$ which is the conditional probability for the risk of failure in interval $t$ given the survival in all previous intervals. The discrete survival function for reaching interval $t$ is:
\begin{align}
S(t) = \Pr(T \geq t) = \prod_{s=1}^{t-1} \left( 1 - h(s) \right) \label{ee2}
\end{align}

Discrete-time data sets are given by $(x_i, \delta_i, t_i), i = 1, \dots, n$, where $t_i = \min(T_i, c_i)$ is the minimum of the survival time $T_i$ and censoring time $c_i$, and $\delta_i$ is  the indicator variable for failure ($\delta_i = 1$) or censoring ($\delta_i = 0$). When $\delta_i = 0$, the $i^{th}$ patient is known to survive until time $c_i$ but the survival time $T_i$ is not observed ($T_i > c_i$).
$x_i$ is a real vector of explanatory variables (e.g., sex, age, weight, etc.) which affect the survival probability. We assume that $x_i$ is inside the unit-sphere,  $\|x_i\| \leq 1$. This is actually a common practice in machine learning. Without loss of generality, we assume that $0 \leq t_i \leq 1$.
For convenience, we use $y_i$ to refer to the term $(2\delta_i - 1)$ and $d_i$ to refer to the tuple $(x_i, y_i, t_i)$.

\section{Discrete-Time Regression Model}

\label{sec3}

In this section, we introduce the discrete-time regression models which are used to model the relationship between explanatory variables and the hazard rate, i.e., the predictive variable. From that, the subsequent sections will discuss the proposed differentially private approaches to guarantee that the estimated parameters from the regression model satisfy the definition of differential privacy.

\subsection{Generalized Linear Models}

We model the effects of explanatory variables $x_i$ to the survival probability by using a generalized linear model:
\begin{align}
g(h(t_i \mid x_i )) = \gamma(t_i) + x_i\p \beta \label{ee1}
\end{align} where $g(\cdot)$ is the link function, $\beta$ is the parameter vector representing the effects of explanatory variables and $\gamma(t_i)$ is a time-varying baseline hazard effect.

A commonly used link function in survival probability is the logit link function
$
g(x) = logit(x) = \log\left(\frac{x}{1-x}\right)
$. The logit link function allows the model to have a nice interpretation of the proportional odds ratio. The other two link functions, which are also used in survival analysis, are the complementary log-log link function
$
g(x) = cloglog(x) = \log(-\log(1 - x)),
$ and the probit link function $g(x) = probit(x) =  \Phi^{-1}(x),$ where $\Phi(\cdot)$ is the cumulative distribution function of the standard normal distribution. Interestingly, the complementary log-log link function has the same interpretation of proportional hazard ratio as the Cox regression. We refer interested readers to \cite{allison1982discrete} for more details.

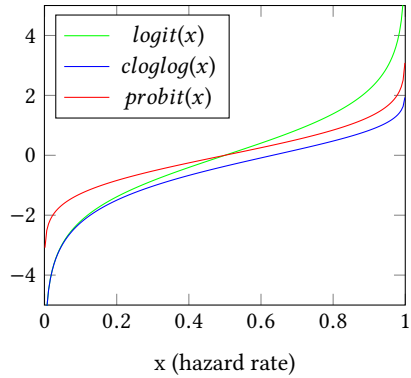
\begin{figure}[t]
  \centering
\begin{tikzpicture}
\begin{axis}[scale=0.7, xmin=-0.0005,xmax=1.00, ymin=-5, ymax=5,legend style={legend pos=north west},
legend entries={ {} {$logit(x)$} , {} {$cloglog(x)$}  , {} {$probit(x)$} }, xlabel={x (hazard rate)},
    declare function={inverf(\x)=\x/abs(\x) * sqrt( sqrt( (4.3307 + ln(1-\x^2)/2 )^2 - ln(1-\x^2)/0.147 ) - (4.3307 + ln(1-\x^2)/2);},
    declare function={normq(\p)=1.4142 * inverf(2*\p-1);}
]
\addplot [domain=0.001:0.999,samples=200,color=green] {ln(x/(1-x))} ;
\addplot [domain=0.001:0.999,samples=200,color=blue] {ln(-ln(1-x)) };
\addplot [domain=0.001:0.999,samples=200,color=red] {normq(x)};

\end{axis}
\end{tikzpicture}

  \caption{The illustration of link functions: logit (green), cloglog (blue) and probit (red). We observe that these link functions are very similar in shape. Notably, the logit link function and cloglog link function are almost identical at $x$ near $0$. This explains why the output models from the logit link function and cloglog link function are very similar in practice when the number of discrete time $q$ is large, or equivalently, the hazard rate $x$ is small. }
  \label {fig:linkfn}
\end{figure}

As illustrated in Figure 1, the three link functions have similar shapes which lead to similar estimation results. In this work, we have selected the logit link function because it has a bounded derivative for the loss function which is required by our proposed Extended Output Perturbation and Extended Objective Perturbation approaches. However, our proposed sampling approach can work with all three link functions.

\subsection{Baseline Hazard Effect}

We model the baseline hazard effect $\gamma(t)$ using natural cubic spline \cite{friedman2001elements} with $e$ knots equally distributed over the interval $[0,1]$, $0 = k_1  < k_2 < \dots <  k_e = 1$.

Let $$d_{j} (t) = \frac{ \max(t-k_j, 0)^3 - \max(t- k_e, 0)^3}{k_e - k_j}$$ and $$
b_1(t) = 1, b_2(t) = t, b_{i+2} = d_i(t) - d_{e-1}(t)
$$ The baseline hazard effect $\gamma(t)$ is approximated as a linear combination of $e$ basis functions:
$$
\gamma(t) = \alpha_1 b_1(t) + \dots + \alpha_e b_e(t)$$     In particular, let $A_i = [b_1(t_i), \dots, b_e(t_i)]\p$ and $\alpha = [\alpha_1, \dots, \alpha_e]\p$, then we can write
$  \gamma(t_i) = \alpha\p A_i.$

\subsection{Maximum Likelihood Estimation (MLE)}

Traditionally, we use MLE to estimate parameters $\alpha$ and $\beta$ in our models.  The aim is to maximize the log-likelihood of the observed data. For simplicity, let
$f = \colvec{2}{\alpha}{\beta}$ and $x_i^t  = \colvec{2}{A_t}{x_i} $.
The log-likelihood function is:
\begin{align*}
\log \Lf (f) &= \sum_{i=1}^n \log\left[ h(t_i \mid f, x_i)^{\delta_i}  (1-h(t_i \mid f, x_i))^{1-\delta_i}  S(t_i \mid f, x_i)\right]
\end{align*}
Let $y_i = 2\delta_i - 1$, from~\eqref{ee2}, \eqref{ee1} and substituting $g(x) = logit(x)$, we can rewrite our problem as:
$$
\log \Lf(f) = - \sum_{i=1}^n \left[\ell_{\textsf{LR}}(y_if\p x_i^{t_i}) + \sum_{s=1}^{t_i-1} \ell_{\textsf{LR}}(-f\p x_i^s)\right]
$$ where $\ell_{\textsf{LR}}(x) = \log(1 + \exp(-x))$ is the logistic loss function.   To further simplify the formula, let $d_i = (x_i, y_i, t_i), i=1, \dots, n,$ and let \begin{align}\ell(f; d_i)  = \ell_{\textsf{LR}}(y_i f\p x_i^{t_i}) + \sum_{s=1}^{t_i-1} \ell_{\textsf{LR}}(-f\p x_i^s)\label{ee3}\end{align} be \emph{the loss function}. Then, we get an ERM problem as follows:
\begin{equation}
f^* = \arg \min_f \sum_{i=1}^n \ell(f;d_i) \label{ff}
\end{equation}

In this work, our main goal is to propose algorithms which protect differential privacy for $f^*$ in Equation~\eqref{ff}.

 	\section{Perturbation Approaches}

\label{sec4}

\subsection{Extended Output Perturbation}

In this section, we  present our proposed algorithm which is
the extension of the Output Perturbation approach in \cite{chaudhuri2011differentially}. For our problem, the loss function is a sum of logistic loss functions instead of a single logistic loss function as in  \cite{chaudhuri2011differentially}.  The proposed algorithm is in fact based on the generalized version of the \emph{Laplace mechanism} \cite{dwork2008differential} which is described as follows: Let $f^*=G(\D)$ be the value that we want to guarantee differential privacy. $f^*$ is the result of applying a function $G$ on the private data set $\D$ (e.g., it is in our case to minimize the objective function).  We define the sensitivity of the function $G$ as follows:
$$
sen(G) = \max_{\D, \D\p} \| G(\D) - G(\D\p)\|
$$ where $\D$ and $\D\p$ are two neighboring data sets. Then, the differentially private version of $f^*=G(x)$ is:
$$
f_{priv} = f^* + \mu
$$ where $\mu$ is a noisy random variable with probability density  function $\pdf(\mu) \propto \exp(-\epsilon \|\mu\| / sen(G))$.

As required by the Output Perturbation approach, we consider the following regularized objective function:
\begin{align}
J(f; \D) = \frac 1 n \sum_{i=1}^{n} \ell(f;  d_i) + \frac \Lambda 2 \|f\|^2
\label{objfunc}\end{align} where $\D = \{d_i\}_{i=1}^n$, $\ell(\cdot)$ is the loss function as defined in~\eqref{ee3} and $\Lambda$ is the regularization parameter. In this approach, our goal is to compute the sensitivity of: $$f^* = \arg \min_f J(f; \D)$$ Then, we use the sensitivity to control the amount of noise added to $f^*$.

\subsubsection{Proposed Algorithm}

\begin{algorithm}[tb]
    \caption{$\A_{\sf Ext-Out-Pert}$: Extended Output Perturbation}
    \begin{algorithmic}[1]
        \REQUIRE        Data set $\D=\{d_1, \dots, d_n\}$, loss function $\ell(f; d_i)$, privacy budget $\epsilon$
        \ENSURE $f_{priv}$
        \STATE $J(f; \D) = \frac 1 n \sum_{i=1}^{n} \ell(f;  d_i) + \frac \Lambda 2 \|f\|^2$
        \STATE Minimize $J(f; \D)$ by using the BFGS algorithm to get the non-private solution $f^*$
        \STATE Compute $t \gets \frac{\sum_{s=1}^{q}\sqrt{4+\|A_s\|^2} + \max_{s \in\{1,\dots,q\}} \sqrt{\|2A_s\|^2+4} }{n\cdot\Lambda}$
        \STATE Sample a random vector $b$ such that $\pdf(b) \propto \exp\left(-\epsilon\frac{\|b\|}{t}\right)$
        \STATE Compute and output $f_{priv} \gets f^* + b$
    \end{algorithmic}
    \label{Algo:OutPert}
\end{algorithm}

Algorithm~\ref{Algo:OutPert} shows the proposed Extended Output Perturbation approach. It returns a vector $f_{priv}$ as the minimizer of $J(\cdot)$ while guaranteeing differential privacy. At Line 2, we compute the non-private solution $f^* = \arg \min_f = J(f; \D)$ using the well-known BFGS algorithm \cite{fletcher2013practical}. $f^*$ is guaranteed to exist due to the strongly convexity of $J(f; \D)$. At Line 3, we compute $t$ which is the sensitivity of $f^*$. Lines 4-5 add noise to the value of $f^*$.

In order to sample a random vector $b$ in Algorithm~\ref{Algo:OutPert} from the distribution $\pdf(b) \propto \exp\left(-\epsilon \|b\| / t \right)$, we observe that the length of the vector $b$ follows a Gamma distribution:
$$
\| b \| \sim \Gamma(d, t / \epsilon)
$$ where $d$ is the number of components of $b$. Thus, in order to sample $b$ we first sample its length $r = \|b\|$ from the Gamma distribution and then sample $b$ as a uniform random point on the surface of a sphere with radius $r$.

\subsubsection{Privacy Guarantee} In order to prove the differential privacy protection, we focus on proving that the sensitivity of $f^*$ at Line 2 in Algorithm~\ref{Algo:OutPert} is equal to the value of $t$ which is computed at Line 3.  Here, we use Lemma~\ref{lem1} from \cite{chaudhuri2011differentially} to bound the sensitivity of $f^*$.

\begin{lem}\label{lem1} Let $G(f)$ and $g(f)$ be two vector-valued functions, which are continuous and differentiable at all points. In addition, let $G(f)$ and $G(f) + g(f)$ be $\lambda-$strongly convex. If $f_1 = \arg\min_f G(f)$ and $f_2 = \arg \min_f G(f) + g(f)$, then
$$    \|f_1 - f_2\| \leq \frac 1 \lambda \max_f \| \nabla g(f)\|
$$
\end{lem}

From Lemma~\ref{lem1}, our goal now is to bound the magnitude of the difference in the gradients of the objective function $J(\cdot)$ on any two neighboring data sets.

\begin{lem} \label{lem2}
For any pair of patient's records $d_i = (x_i, y_i, t_i)$ and $d_j=(x_j, y_j, t_j)$, and for any $f$,
$$
\|\nabla \ell(f; d_i) - \nabla \ell(f; d_j)\| \leq \sum_{s=1}^{q}\sqrt{\|A_s\|^2 + 4} + \max_{s \in \{1, \dots, q\}} \sqrt{ \|2A_s\|^2 + 4 }
$$
\end{lem}
\begin{proof}
\begin{align*}
\nabla \ell(f; d_i)
&= \nabla \ell_{\textsf{LR}}(y_i f'x_i^{t_i}) + \sum_{s=1}^{t_i-1} \nabla \ell_{\textsf{LR}}(-f\p x_i^s)\\
&= \frac{-y_i x_i^{t_i}}{1+\exp(y_i f\p x_i^{t_i})} +
\sum_{s=1}^{t_i-1}
\frac{x_i^{s}}{1+\exp(-f\p x_i^{s})}
\end{align*}

Therefore, we can write $\nabla\ell(f; d_i) = \sum_{s=1}^q l_i^s,$ where
$$
l_i^s=
\begin{cases}
\frac{x_i^s}{1+\exp(-f\p x_i^s)}, & if~s < t_i\\
\frac{-y_i x_i^s}{1+\exp(y_i f\p x_i^s)}, & if~s = t_i\\
\vec{0}, & if~s > t_i
\end{cases}
$$

Similarly, we can also write $\nabla\ell(f; d_j) = \sum_{s=1}^q l_j^s.$
Therefore,
$$\nabla\ell(f; d_i) - \nabla\ell(f; d_j) = \sum_{s=1}^q l_i^s - l_j^s$$

We have $|\frac{-y_i}{1+\exp(y_i f\p x_i^s)}|~\leq 1$, $\|x_i\|~\leq 1, \|x_j\|~\leq 1$,  for any $s \in \{1 \dots q\}$, we consider four possible cases as follows:

\textbf{Case 1}: if $s < t_i$ and $s < t_j$,  then \begin{align*}\|l_i^s - l_j^s\| = \left\| \colvec{2}{(e_1-e_2) A_s}{e_1x_i - e_2x_j} \right\| &\leq \sqrt{ \|A_s\|^2 + (\|x_i\| + \|x_j\|)^2} \\ &\leq \sqrt{\|A_s\|^2 + 4}\end{align*} where $e_1 = \frac 1 {1 + \exp(-f \p x_i^s)}$ and $e_2 = \frac 1 {1 + \exp(-f\p x_j^s)}.$

\textbf{Case 2}: if $s > t_i$ or $s > t_j$, then $\|l_i^s - l_j^s| \leq \max(\|x_i^s\|, \|x_j^s\|) \leq \sqrt{\|A_s\|^2 + 1} < \sqrt{\|A_s\|^2+4}.$

\textbf{Case 3}: if $l_i^s = \frac{-x_i^s}{1+\exp(f \p x_i^s)}$ and $l_j^s = \frac{x_j^s}{1+\exp(-f \p x_j^s)}$, then $$\|l_i^s - l_j^s\| = \left\|-\colvec{2}{(e_1+e_2)A_s}{e_1 x_i + e_2 x_j}
\right\| \leq \sqrt{\|2A_s\|^2 + 4}$$  where $e_1 = \frac 1 {1 + \exp(f \p x_i^s)}$ and $e_2 = \frac 1 {1 + \exp(-f\p x_j^s)}$.

\textbf{Case 4}: if $l_i^s = \frac{x_i^s}{1+\exp(-f \p x_i^s)}$ and $l_j^s = \frac{-x_j^s}{1+\exp(f \p x_j^s)}$, then $\|l_i^s - l_j^s\| \leq \sqrt{\|2A_s\|^2 + 4}.$ This case is similar to \textbf{Case 3}.

We observe that there is at most one value of $s, 1 \leq s \leq q,$ belonging to \textbf{Case 3} or \textbf{Case 4} in which $\|l_i^s - l_j^s\| \leq \sqrt{\|2A_s\|^2 + 4}$. Therefore, from the triangle inequality:
$$
\|\sum_{s=1}^q l_i^s - l_j^s \| \leq \sum_{s=1}^{q}\sqrt{\|A_s\|^2 + 4} + \max_{s \in \{1, \dots, q\}} \sqrt{ \|2A_s\|^2 + 4 }$$
Therefore, the lemma follows.\end{proof}

Finally, we can bound the sensitivity of $f^*=\arg\min_f J(f; \D)$ by the following lemma.

\begin{lem} \label{coro1}
The $\ell_2-$sensitivity of $f^*=\arg\min_f J(f; \D)$ is at most $\frac{\sum_{s=1}^{q} \sqrt{4  + \|A_s\|^2}   + \max_{s \in \{1, \dots, q\}} \sqrt{ \|2A_s\|^2 + 4 }    }{n\Lambda}$.
\end{lem}

\begin{proof}
    Without loss of generality, we assume that two neighboring data sets $\D$ and $\D\p$ are different at $n^{th}$ patient
    with $(x_n, y_n, t_n) \in \D$ and $(x\p_n, y\p_n, t\p_n) \in \D\p$.
    
    Let $G(f) = J(f; \D)$, $g(f) = J(f; \D\p) - J(f; \D) = \frac 1 n ( \ell(f; d\p_n) - \ell(f; d_n) )$, $f_1 = \arg\min_f J(f; \D)$, and $f_2 = \arg\min_f J(f; \D\p)$.
    Because $\frac 1 2 \|f\|^2$ is $1-$strongly convex, $G(f) = J(f; \D)$ is $\Lambda-$strongly convex and $G(f) + g(f) = J(f; \D\p)$ is also $\Lambda-$strongly convex. From Lemma~\ref{lem2},
    \begin{align*}
    \|\nabla g(f)\| &= \left\|\frac 1 n \left(  \nabla\ell(f; d\p_n) - \nabla \ell(f; d_n)  \right) \right\|
    \\&\leq \frac{\sum_{s=1}^{q}\sqrt{4+\|A_s\|^2}  + \max_{s \in \{1, \dots, q\}} \sqrt{ \|2A_s\|^2 + 4 } }{n}
    \end{align*}
From Lemma~\ref{lem1},
$$
\|f_1 - f_2 \| \leq \frac 1 \Lambda \frac {\sum_{s=1}^{q}\sqrt{4+\|A_s\|^2}  + \max_{s \in \{1, \dots, q\}} \sqrt{ \|2A_s\|^2 + 4 }}{n}
$$
Therefore, the lemma follows.
\end{proof}

\begin{theorem} Algorithm~\ref{Algo:OutPert} is $\epsilon-$differentially private.
\end{theorem}
\begin{proof}
    For any pair of neighboring data sets $\D$ and $\D\p$ and for any $f_{priv}$,
    $$
    \frac{\pdf(f_{priv}\mid \D)}{\pdf(f_{priv}\mid \D\p)} = \frac{\pdf(b_1)}{\pdf(b_2)} = \exp\left(-\epsilon/t(\|b_1\| - \|b_2\|)\right)
    $$ where $b_1$ and $b_2$ are the corresponding noise vectors at Line~4 in Algorithm~\ref{Algo:OutPert} with respect to the data sets $\D$ and $\D\p$. If $f_1^*$ (resp., $f_2^*$) is the solution at Line~2 of Algorithm~\ref{Algo:OutPert} on the data set $\D$ (resp., $\D\p$), then $f^*_1 + b_1 = f^*_2 + b_2 = f_{priv}$. From Lemma~\ref{coro1} and the triangle inequality:
    $$
    \|b_1\| - \|b_2\| \leq \|b_1 - b_2\| = \|f_1 - f_2\| \leq t
    $$ where $t = \frac{\sum_{s=1}^{q}\sqrt{4+\|A_s\|^2} + \max_{s \in\{1,\dots,q\}} \sqrt{\|2A_s\|^2+4} }{n\cdot\Lambda}.$
    Therefore, $\frac{\pdf(b_1)}{\pdf(b_2)} \leq \exp(\epsilon)$. Thus, Algorithm~\ref{Algo:OutPert} is $\epsilon-$differentially private.
\end{proof}

\subsection{Extended Objective Perturbation}

In this section,  we present a solution based on the Objective Perturbation approach proposed in \cite{chaudhuri2011differentially}.
Similarly to the Extended Objective Perturbation approach, we also consider the objective function as described in Equation~\eqref{objfunc}.
In this approach, instead of adding noise to the solution of the optimization problem as the output perturbation does, it adds noise to the objective function.

\subsubsection{Proposed Algorithm}

\begin{algorithm}[htb]
	\caption{$\A_{\sf Ext-Obj-Pert}$: Extended Objective Perturbation}
	\begin{algorithmic}[1]
		\REQUIRE		Data set $\D=\{ d_1,\dots, d_n\}$, objective function $J(f; \D)$, privacy budget $\epsilon$, parameter $\Lambda$
		\ENSURE $f^*$
        \STATE $\Delta \gets 0$
		\STATE Compute $\epsilon\p \gets \epsilon - 2\sum_{s=1}^{q}\log\left(1 + \frac{\frac 1 4\sqrt{\|A_s\|^2+1}}{n(\Lambda+\Delta)}\right)$
		\IF{$\epsilon\p < \epsilon/2$}
		\STATE Binary search value of $\Delta$ such that
		$
		2\sum_{s=1}^{q}\log(1 + \frac{\frac 1 4\sqrt{\|A_s\|^2 +1}}{n\left(\Lambda+\Delta\right)})
		= \epsilon / 2
		$ 		 and set $\epsilon\p \gets \epsilon/2$
		\ENDIF
		\STATE Compute $t \gets \sum_{s=1}^{q}\sqrt{4+\|A_s\|^2} + \max_{s\in\{1,\dots, q\}}\sqrt{\|2A_s\|^2 + 4}$
		\STATE Sample a random vector $b$ such that $\pdf(b) \propto \exp\left(-\epsilon\p\|b\|  / t  \right)$
		\STATE $f^* \gets \arg\min_f J(f; \D) + \frac 1 n \langle b, f\rangle + \frac 1 2 \Delta \|f\|^2$
		\STATE Output $f^{*}$
	\end{algorithmic}
	\label{algo2}
\end{algorithm}

Algorithm~\ref{algo2} shows the solution in pseudo-code. At Line 2, we compute $\epsilon\p$ which is used to calibrate the magnitude of a random variable $b$. Here, the regularization parameter is equal to $\Lambda$. At Line 3, if $\epsilon\p < \epsilon / 2$, then it indicates that $\Lambda$ is not large enough. In this case, an additional positive regularization parameter $\Delta$ is picked to set the value of $\epsilon\p$ equals to $\epsilon/2$ (Line 4). At Line 5, we compute $t$ which is the sensitivity of $\nabla J(f; \D)$. Line 6 samples a random vector $b$ using the same method described in Subsection 4.1.1. Lines 7-8 return the solution of the noisy objective function using the BFGS algorithm.

\subsubsection{Privacy Guarantee} In this section, we will prove that the probability density of $f^*$ from Algorithm~\ref{algo2} satisfies the differential privacy definition.

\begin{thm}
    Algorithm 2 is $\epsilon-$differentially private.
\end{thm}
\begin{proof}

The noisy objective function from Algorithm 2 is:
$$
f^* = \arg \min_f \frac 1 n \sum_{i=1}^n \ell(f; d_i)  + \frac 1 n \langle b, f\rangle + \frac 1 2 \left(\Lambda + \Delta\right) \|f\|^2
$$

Due to the convexity of $\ell(\cdot)$, the gradient is zero at the minimal point $f^*$, equivalently,
$$
b = -n(\Lambda+\Delta) f^*  - \sum_{i=1}^n \nabla \ell(f^*; d_i)
$$ Due to the strongly convexity of the objective function, there is a bijective (injective and surjective) mapping from $f$ to $b$ (denoted as $f \to b$). Therefore, we can transform the probability density function of random variable $f$ to the probability density function of random variable $b$ by a multiplication factor of the Jacobian determinant \cite{billingsley2008probability}. From that, the probability density ratio in differential privacy can be rewritten as:
\begin{align}
\frac{\pdf(f\mid\D)}{\pdf(f\mid\D\p)} = \frac{\pdf(b\mid \D)}{\pdf(b\p\mid \D\p)} \cdot \frac{|\det\left(Jacob\left(f \to b \mid \D \right)\right)|^{-1}}{|\det\left(Jacob\left(f \to b\p\mid \D\p \right) \right)|^{-1}}\label{eqn4}
\end{align}

We first bound the ratio of the Jacobian determinants. Without loss of generality, we assume that the two data sets $\D$ and $\D\p$ are different at $n^{th}$ record with $d_n \in \D$ and $d_{n\p} \in \D\p$.
Let
$$
A = -Jacob(f\to b \mid \D) = n(\Lambda + \Delta) \mathbb{I}+ \sum_{i=1}^n \nabla^2 \ell(f^*; d_i)
$$ and
$
E =  \nabla^2 \ell(f^*; d_n) - \nabla^2 \ell(f^*; d\p_n)
$, then $$  \frac{|\det\left(Jacob\left(f \to b\mid \D \right) \right)|^{-1}}{|\det\left(Jacob\left(f \to b\p \mid \D\p \right)\right)|^{-1}}
 =
 \frac{|det(A+E)|}{|det(A)|} = |det(\mathbb I + A^{-1}E)|
 $$ Moreover, $
E = \sum_{s=1}^q E_n^s - \sum_{s=1}^q E_{n\p}^s
$
where
$$
E_n^s=
\begin{cases}
\frac{(x_n^s) (x_n^s)\p}{(1+\exp(f\p x_n^s))(1+\exp(-f\p x_n^s))}, & if~s < t_n\\
\frac{-y_n^2(x_i^s) (x_n^s)\p}{(1+\exp(y_n f\p x_n^s))(1+\exp(-y_n f\p x_n^s))}, & if~s = t_n\\
0, & if~s > t_n
\end{cases}
$$
Similarly, we can define $E_{n\p}^s$ by replacing $n$ by $n\p$.  From  \cite{seiler1975inequality}, for any square matrices $A$ and $B$, $$det(\mathbb I + A + B) \leq det(\mathbb I+|A|)\cdot det(\mathbb I + |B|)$$ where $|A| = (A\p A)^{\frac 1 2}$. Moreover, $A^{-1}E_n^s$ and $A^{-1}E^s_{n\p}$ are symmetric, thus
$$
det(\mathbb I + A^{-1}E) \leq \prod_{s=1}^{q} det(\mathbb I + A^{-1}E_n^s) \cdot det(\mathbb I + A^{-1}E_{n\p}^s)
$$
We now prove that $|det\left(\mathbb I + A^{-1}E_n^s\right)| \leq 1 + \frac{\frac 1 4 \sqrt{\|A_s\|^2 + 1} }{n(\Lambda + \Delta)}$. Because
$\left|\frac{-y_n^2}{(1+\exp(y_n f\p x_n^s))(1+\exp(-y_n f\p x_n^s))}\right| \leq \frac 1 4$, and
$E_n^s$ is either a zero matrix or 1-rank matrix. The only non-zero eigenvalue of $E_n^s$ if exist satisfies $|\lambda_1(E_n^s)| \leq \frac 1 4 \| x_n^s\| \leq \frac 1 4 \sqrt{\|A_s\|^2 + 1}$. As the objective function is $(\Lambda + \Delta)-$strongly convex, $A$ is a full-rank matrix with each eigenvalue greater than $n(\Lambda + \Delta)$. Therefore, $|det\left(\mathbb I + A^{-1}E_n^s\right)| \leq 1 + \frac{\frac 1 4\sqrt{\|A_s\|^2 + 1} }{n(\Lambda + \Delta)}.$ Similarly, $|det(\mathbb{I} + A^{-1}E_{n\p}^s)| \leq 1 + \frac{\frac 1 4\sqrt{\|A_s\|^2 + 1} }{n(\Lambda + \Delta)}.$ Therefore, \begin{align}
 \frac{|\det\left(Jacob\left(f \to b\right) \mid \D \right)|^{-1}}{|\det\left(Jacob\left(f \to b\p\right) \mid \D\p \right)|^{-1}}
 \leq \exp \left( 2\sum_{s=1}^{q}\log(1 + \frac{\frac 1 4\sqrt{\|A_s\|^2+1}}{n\Lambda}) \right) \label{eqn1}\end{align}

Next, we bound the ratio of the probability density of random vector $b$ with respect to two neighboring data sets.  We have:
$$b - b\p = \nabla \ell(f^*; d_n) - \nabla \ell(f^*; d\p_n)$$ From Lemma~\ref{lem2}, $$\|b - b\p \| \leq \sum_{s=1}^{q} \sqrt{\|A_s\|^2 + 4} + \max_{s \in\{1, \dots, q\}} \sqrt{\|2A_s\|^2+ 4}$$ Therefore,
\begin{align}
\frac{\pdf(b\mid \D)}{\pdf(b\p\mid \D\p)} \leq \exp(\epsilon\p \|b-b\p\| / t) \leq \exp(\epsilon\p)\label{eqn2}
\end{align}

From \eqref{eqn4}, \eqref{eqn1},  \eqref{eqn2}, and $\epsilon = \epsilon\p +  2\sum_{s=1}^{q}\log\left(1 + \frac{\frac 1 4\sqrt{\|A_s\|^2+1}}{n\Lambda}\right)$, the theorem follows.
\end{proof}

\section{Proposed Sampling Approach}
\label{sec5}

In this section, we propose a solution which guarantees differential privacy by directly sampling a random output from a modified version of the posterior distribution. In this work, we pick a normal distribution as the prior distribution. This is equivalent to using:
$$
\mathcal U(f; \D) = -\frac 1 2 \sigma \|f \|^2 - \sum_{i=1}^n \ell(f; d_i)
$$ as the utility function in the \emph{exponential mechanism} \cite{mechanism-design-via-differential-privacy} where the parameter $\sigma$ is used to control the variance of the prior normal distribution.
Then, the differrentially private output is sampled from the following distribution:
$$
\pdf(f) \propto \exp\left(\frac {\epsilon \UU(f;\D)}{  2\Delta_\UU}\right)
$$ where $\Delta_\UU = \max_{d(\D, \D\p)=1, f} \|\UU(f;\D) - \UU(f; \D\p)\|$ is the sensitivity of $\UU.$
The reason we pick a normal prior distribution instead of a uniform prior distribution is not because our proposed solution required so to guarantee differential privacy but we observe that with a normal prior distribution the sampling algorithm converges better and is more stable.

Moreover, this approach requires the utility function $\UU(f; \D)$ has to have a bounded sensitivity. However, the loss function $\ell(\cdot)$ is not bounded. Therefore, the function $\UU(f; \D)$ has unbounded sensitivity. In order to overcome this difficulty, we propose a smooth sanitizer function $C(x)$ which is used to control the maximum value of the loss function $\ell(\cdot)$. The definition of $C(x)$ is given as follows:$$
C_v(x )=
v \cdot tanh\left(\frac x v\right)
$$ which is illustrated in Figure~\ref{figure2}.  We now take the composition of $C_v(\cdot)$ with $ \ell(f; d_i)$ to have a bounded-sensitivity utility function:
$$
\UU(f; \D) =  -\frac 1 2 \sigma \|f \|^2  - \sum_{i=1}^n C_v\left(
\ell(f; d_i) \right)
$$

 We intentionally pick the $tanh(\cdot)$ function as the sanitizer because it nicely keeps the loss function in its original form when the value of the loss function is near $0$. Meanwhile, it deforms the loss function at large values to make the function finite. The advantage of this approach is that the sampled parameter can arbitrary large while the objective function is kept almost the same around the optimal parameter which maximizes the posterior probability.
 We describe the pseudo-code of our approach in Algorithm~\ref{algo3}.

\begin{figure}[t]
  \centering
\begin{tikzpicture}
\begin{axis}[scale=0.7, xmin=0,xmax=40, ymin=0, ymax=40,legend style={legend pos=north west},
legend entries={ {} {$y = x$} , {} {$y = C_{20}(x)$}   }, xlabel={x},
  declare function={
    blossfn(\x)= (x<=19)*x + (x>=19)*( 20- 1/(x-18)) ;
  },  declare function={
    blossfn1(\x)= (x<=19)*x + (x>=19)*(20-exp(19-x)) ;
  }
]
\addplot [domain=0:40,samples=200,color=red] {x} ;
\addplot [domain=0:40,samples=200,color=blue] { 20*tanh(x/20) };

\addplot [domain=0.0:40,samples=200,color=black,dashed] {20 };

\end{axis}
\end{tikzpicture}

  \caption{The illustration of the sanitizer function (blue) with maximum value $20$ and the identity function (red).}
  \label {figure2}
\end{figure}
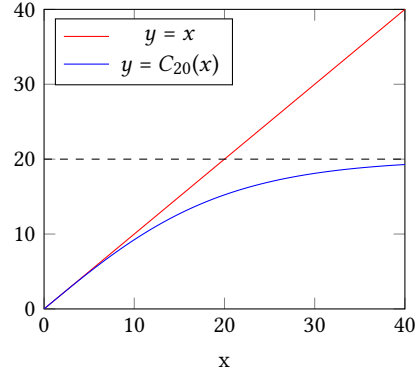

 \begin{algorithm}[tb]
    \caption{$\A_{\sf Sanitized-EXP}$: Sanitized Loss Mechanism}
    \begin{algorithmic}[1]
        \REQUIRE        Data set $\D=\{d_i \}_{i=1}^n$, loss function $\ell(f; d_i)$, privacy budget $\epsilon$, maximum value $v$, parameter $\Lambda$
        \ENSURE $f$
        \STATE $\UU(f; \D) =  -\frac 1 2 \sigma \|f \|^2  - \sum_{i=1}^n C_v(\ell(f; d_i))$
        \STATE Sample a random vector $f$ with the probability density $$\pdf(f) \propto \exp\left(\frac{\epsilon}{2v} \UU(f; \D)\right)$$
    \end{algorithmic}
    \label{algo3}
\end{algorithm}

\newpage
\begin{thm}[Privacy guarantee] Algorithm~\ref{algo3} is $\epsilon-$differentially private.
\end{thm}
\begin{proof}
    For two neighboring data sets $\D$ and $\D^\prime$, \\ $\Delta_\UU = \max_{f, d(\D, \D\p)=1}|\UU(f; \D) - \UU(f; \D\p)| \leq v $.
    Therefore, at any point $f,$ we have \begin{align*}\frac{\pdf(f\mid \D)}{\pdf(f\mid \D\p)}
    & =
    \frac{\exp\left(\frac {\epsilon} {2v} \UU(f; \D)\right) /_{
    \int    \exp\left(\frac {\epsilon} {2v} \UU(f; \D)\right) df}
    }{\exp\left(\frac {\epsilon} {2v} \UU(f; \D\p)\right) /_{
        \int    \exp\left(\frac {\epsilon} {2v} \UU(f; \D\p)\right) df}
}\\
    &\leq \exp\left(\frac {2\epsilon} {2v}  \left| \UU(f; \D) - \UU(f; \D\p)\right|\right)\\
    &\leq \exp(\epsilon)
    \end{align*}
Therefore, Algorithm~\ref{algo3} is  $\epsilon-$differentially private.
\end{proof}

The problem with Algorithm~\ref{algo3} is that there is no run-time efficient algorithm to sample the distribution of $f$ exactly. Bassily et al. \cite{bassily2014private} proposed a polynomial run-time sampling algorithm. However, their proposed algorithm is still impractical due to the high degree of the polynomial run-time complexity and only apply for the log-convex function.
Recently, there are developments \cite{chen2014stochastic,ahn2012bayesian,ma2015complete} in Markov Chain Monte Carlo (MCMC) method which can be applied to machine learning problems with large data sets. The idea is to construct Markov chains to simulate dynamical systems with stochastic gradients. At each step, we compute the gradient at the current location, then add a controlled amount of noise to the gradient and follow the noisy gradient to a new location. Asymptotically, the stationary distribution of this process converges to the true distribution from which the gradient is computed.

In this work, we propose to use an MCMC sampling algorithm, namely Preconditioned Stochastic Gradient Langevin Dynamics (pSGLD) \cite{li2015preconditioned}, to approximately sample the posterior distribution.
pSGLD is good at sampling variables with differences in scale which is useful for our problem because the parameter $\alpha$ is usually much larger in magnitude than the parameter $\beta$ (recall that $f=[\alpha, \beta]\p$).
The pseudo-code of pSGLD is described in Algorithm~\ref{algo4}.
At Line 1, we initialize the values of $V_0$ and $f_1$. Line 3 computes the learning rate $\epsilon_t$. It is required that $\lim_{t \to \infty} \sum_t \epsilon_t \to \infty$ and $\lim_{t \to \infty} \sum_t \epsilon_t^2 < \infty$ to guarantee the convergence. We sample uniformly $k$ records from $\D$ for estimating the average gradient $\bar g^t$ (Line 5). We then compute the variance of the gradient at Line 6 ($\odot$ is the element-wise product) and convert it to the preconditioned matrix $G^t$ at Line 7. We update the parameter at Line 8 with a noise variable $\mathcal N(0, \epsilon_t G^{t})$.
It is worth to note that there is a permanent bias in \textsf{pSGLD} due to excluding a correction term in the updating step (Line 8). However, this bias is negligible and excluding the correction term helps to speed up the sampling algorithm which then helps to reduce the finite-sample bias as more steps are executed in a finite amount of time.

\begin{algorithm}[t]
    \caption{$\A_{\sf pSGLD}$: pSGLD Sampling Algorithm}
    \begin{algorithmic}[1]
        \REQUIRE Data set $\D=\{d_i\}_{i=1}^n$, loss function $\ell$, privacy parameter $\epsilon$, $\mu$, $k$, bounded value $v$ and learning rate $\tau$
        \ENSURE $f^{T+1}$
        \STATE $V_0 \gets \vec{0}$, $f_1 \gets \vec{0}$
        \FOR{$t = 1$ to $T$}
            \STATE Compute $\epsilon_t \gets t^{-\tau}$
            \STATE Uniformly sample $\Omega_k^t = \left\{d_{t_1}, \dots, d_{t_k}\right\} \subset \D$
            \STATE Compute $\bar g^t = \frac \epsilon {2v} (\frac {\sigma f^t} n  + \frac 1 {k}  \sum_{i=1}^k \nabla C_v(\ell(f^t, d_{t_i}))) $
            \STATE $V^t \gets \mu V^{t-1} + (1-\mu)(\bar g^t \odot \bar g^t)$
            \STATE $G^t \gets 1 / \left( \lambda \mathbb{I} + diag(\sqrt{V^t})\right)$
            \STATE $f^{t+1} \gets f^{t} - \epsilon_t  \left(G^{t} \cdot n \bar g^t  \right) + \mathcal{N}(0, \epsilon_t G^{t})$
        \ENDFOR
        \STATE Output $f^{T+1}$
    \end{algorithmic}
    \label{algo4}
\end{algorithm}

\section{Experimental Evaluation}
\label{sec6}

In this section, we present the results of our experiments on four real data sets. We focus on answering the following three important research questions: (1) Does the sampling approach converge to its stationary distribution? (2) What is the trade-off between privacy and accuracy as compared to the non-private estimation? (3) Are the discrete-time regression models good alternatives to the Cox regression model? In the following sections, we address the above research questions accordingly.

\subsection{Data Sets}

\begin{table}[h]

        \renewcommand{\arraystretch}{1.3}
        \caption{Statistics of the data sets.}
        \centering
        \begin{tabular}{ |c| | c|c |c| }
                \hline
                { Data set } & {Size} & {\#uncensored} & {\#explanatory variables} \\ \hline\hline
                \textsf{FL} & $7874$ & $2169$ & $8$ \\          \hline
                \textsf{TB} & $16116$ & $1761$ & $3$\\                     \hline
                \textsf{WT} & $21685$ & $18615$ & $3$\\          \hline
                \textsf{SB} & $53558$ & $16341$ & $3$\\
                \hline
        \end{tabular}

        \label{tab:dataset}
\end{table}

We use four real data sets in our experiments. Table~\ref{tab:dataset} gives the statistics of these data sets.

\begin{itemize}
    \item \header{\it The FLchain data set}  (\textsf{FL}) - It is obtained from a study on the association of the serum free light chain with higher death rates \cite{kyle2006prevalence, dispenzieri2012use}. The survival time of a patient is measured in days from enrollment until death. The censored cases are patients who are still alive at the last contact. The explanatory variables are \textsf{age}, \textsf{sex}, \textsf{creatinine}, \textsf{mgus}, etc.

    \item \header{\it The time-to-second-birth} (\textsf{SB})  and \textit{time-to-third-birth} (\textsf{TB}) \textit{data sets} - They are obtained from The Medical Birth Registry of Norway \cite{irgens2000medical}. The survival time is the time between the first and second births,  and between the second and third births respectively. The censored cases are women who do not have the second birth, and the third birth respectively, at the time the data are collected. The explanatory variables in \textsf{SB} (resp., \TB) are \textsf{age}, \textsf{sex} and death
 of earlier children (resp., \textsf{age}, \textsf{spacing} and \textsf{sibs}).

\item \header{\it The Wichert data set}  (\textsf{WT}) - It contains records on unemployment duration of people in Germany \cite{wichert2008simple}. The survival time is the duration of unemployment until having a  job again. The censored cases are the ones who do not have a new job at the time the data are collected. The explanatory variables are \textsf{sex}, \textsf{age} and \textsf{wage}.

\end{itemize}

The survival times in these four data sets are normalized to the interval $[0, 1]$. We set the number of discrete-time intervals $q = 200$. All the vectors of the explanatory variables are normalized to have zero mean and fitted inside the unit sphere. We use the natural cubic spline with $e = 3$ knots to model the baseline hazard effect.

\subsection{Convergence of the Proposed Sampling Approach}

\begin{figure*}[tp]
    \centering
    \includegraphics[width=\textwidth]{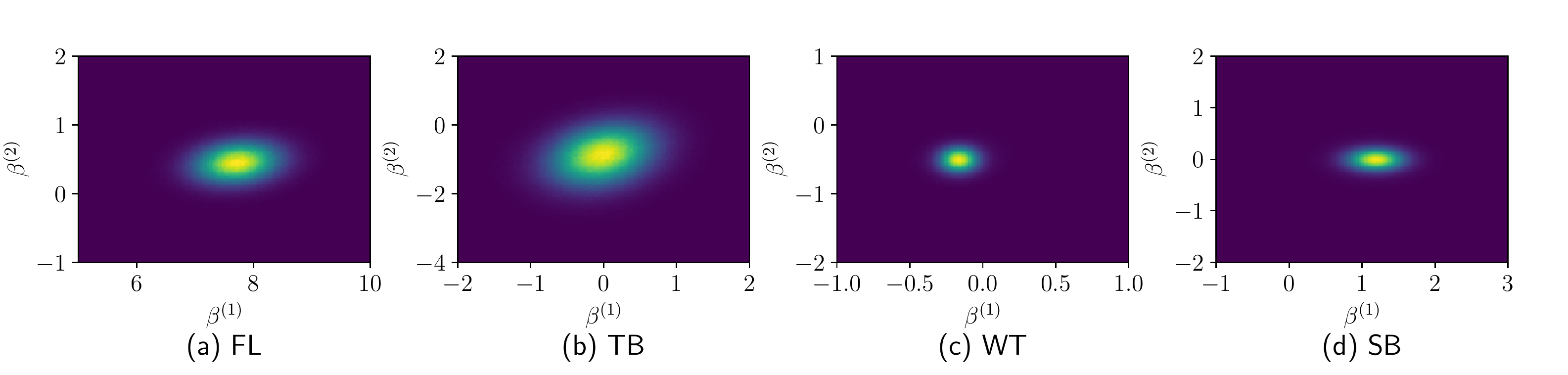}
    \caption{An illustration of the probability densities of the sampling posteriors after 250 epochs at privacy budget $\epsilon = 6.4$.}\label{figure1}
\end{figure*}

\begin{figure}[tb]
    \centering
        \includegraphics[width=0.5\linewidth]{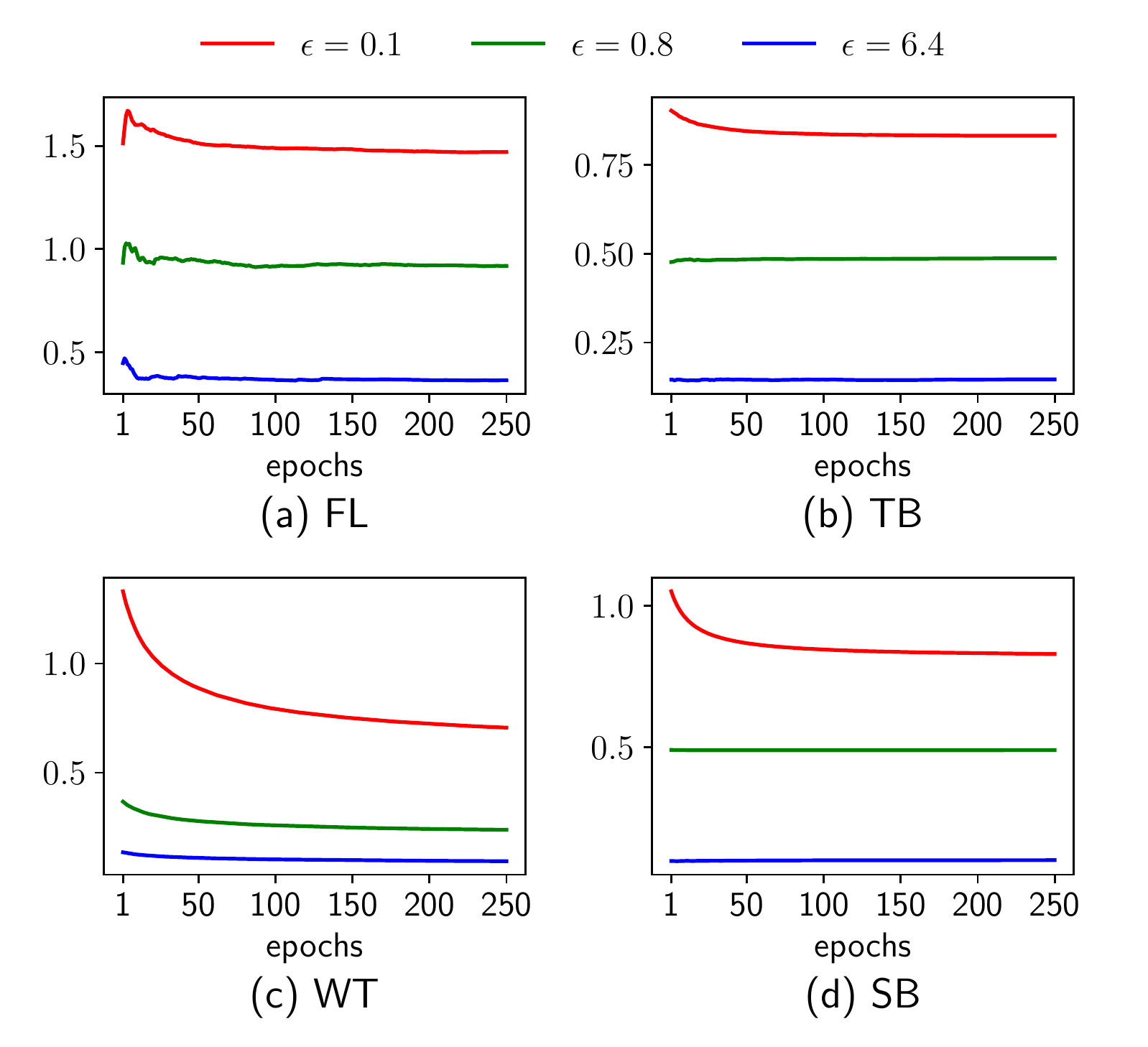}
    \caption{An illustration of MRE as a statistical test for samples from the pSGLD sampling algorithm.}\label{figure4}
\end{figure}

This section reports on the convergence of our proposed sampling approach. The aim is to check whether it converges to the stationary distribution.  The loss function is bounded by the value $v = 2\log(n)$ where $n$ is the size of the data set. We set the parameter $\sigma = 10^{-2}\cdot2v/\epsilon$. At each step of the Markov chain, we randomly pick $k = 200$ records from the data set to compute the gradient. We set the parameters $\tau = 0.51$, $\lambda = 10^{-5}$ and $\mu = 0.99$ in Algorithm~\ref{algo3}. In Figure~\ref{figure1}, we plot the estimated probability densities of the two first parameters ($\beta^{(1)}$ and $\beta^{(2)}$) after 250 epochs from the sampling process. We remove the first $10^4$ steps as the Markov chain does not reach the stationary distribution at the beginning. We can observe that the probability densities of the samples are very similar to the normal distributions which are actually what we expect when sampling from the posterior distributions.

For a more formal test, we use the mean relative error (MRE) as a statistical test of convergence. MRE is defined as follows: \begin{align}\label{eqn:mre}
    \operatorname{MRE} = \frac{1}{t} \sum_{i=1}^t \frac{\|f_i - f^*\|}{\|f^*\|}
\end{align} where $f_i$ is the parameter vector from the sampling process, $f^*$ is the optimal parameter vector which maximizes the likelihood in non-private setting and $t$ is the number of samples. We plot the MRE as the function of epochs with three different privacy budgets in Figure~\ref{figure4}. Each epoch is a bundle of $n$ steps. We observe that after 250 epochs, MRE becomes stable which indicates that the sampling procedure converges to its stationary distribution.

\subsection{Trade-off between Privacy and Accuracy}

\begin{table}[!ht]
        \renewcommand{\arraystretch}{1.3}
        \caption{The performance in MRE of \textsf{Ext-Out-Pert} approach for different regularization parameters with privacy budget $\epsilon = 6.4$. The best performance results are in bold.}
        \centering
        \begin{tabular}{c|| c c c c  c c c }
          $\Lambda$ & $10^{-5}$ & $10^{-4}$ & $10^{-3}$ & $10^{-2}$  & $10^{-1}$ & $10^{0}$  \\
          \hline
 \FL   &  981.635  &  98.135  &  9.828  &  1.195  &  \textbf{0.837}  &  0.882   \\
\TB  &  90.994  &  9.113  &  1.273  &  \textbf{0.964}  &  0.975  &  0.983\\
\WT    &  342.939  &  34.297  &  3.456  &  \textbf{0.763}  &  0.765  &  0.81 \\
\SB  &  9.59  &  1.224  &  \textbf{0.957}  &  0.993  &  0.998  &  0.999  \\
        \end{tabular}

        \label{table3}
\end{table}

In this section, we investigate the trade-off between privacy and accuracy in our proposed approaches. We first need to pick the value of regularization terms for the perturbation approaches (\textsf{Ext-Obj-Pert} and \textsf{Ext-Out-Pert}) as the accuracy of these approaches are very much depend on the regularization parameter $\Lambda$. We report in Table~\ref{table3} the MREs of \textsf{Ext-Out-Pert} with different values of $\Lambda$ and privacy budget $\epsilon = 6.4$. For consistency in performance comparison, we will use the best values of $\Lambda$, which lead to the smallest relative error per data set.

\begin{figure*}[!t]
    \centering
        \includegraphics[width=\textwidth]{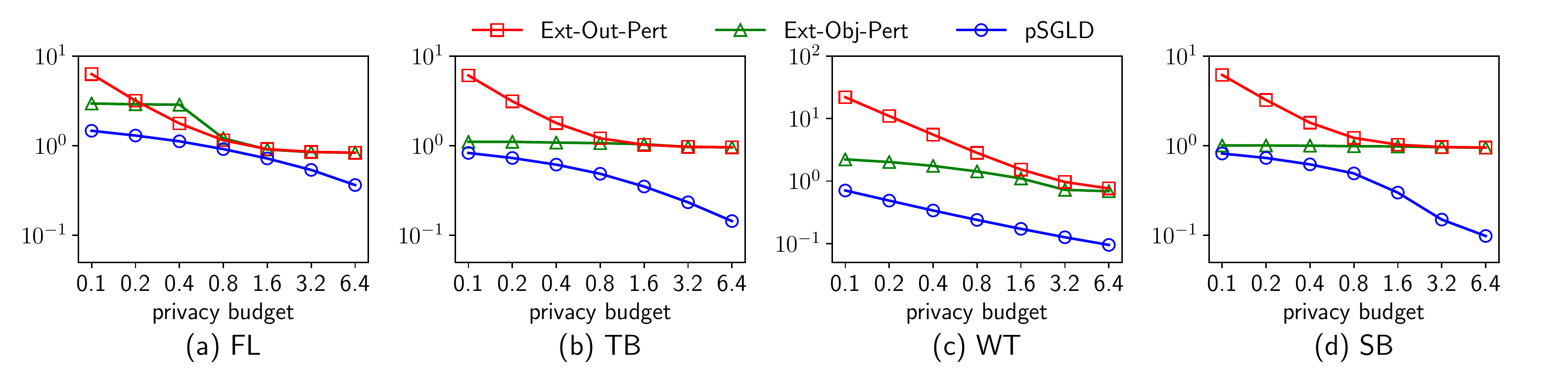}
    \caption{The performance of our proposed approaches in MRE with privacy budget $\epsilon$ from $0.1$ to $6.4$.}\label{figure5}
\end{figure*}

To measure the accuracy of the proposed approaches at different privacy levels, the privacy budget is varied from $0.1$ to $6.4$. We also use MRE for the measurement. The results are shown in Figure~\ref{figure5}. Overall, \textsf{pSGLD} outperforms both \textsf{Ext-Out-Pert} and \textsf{Ext-Obj-Pert} approaches. Moreover, we observe that the accuracy of \text{Ext-Out-Pert} and \textsf{Ext-Obj-Pert} does not improve much at high privacy budgets. It is due to the large regularization parameter that causes the output parameter moving towards the zero vector instead of the optimal parameter as the regularization term is the dominant factor of the objective function.  Meanwhile, our proposed sampling approach (\textsf{pSGLD}) does not suffer from this effect which leads to much better results at high privacy budgets.

\subsection{Comparison with Cox regression}
\begin{table}[!pth]
        \renewcommand{\arraystretch}{1.3}
        \caption{Relative error of the discrete-time survival regression as compared to the Cox regression.}
        \centering
        \begin{tabular}{ |c| |  c | }
                \hline
                { Data set } & Relative error (\%) \\ \hline\hline
                \textsf{FL} & $2.589\%$  \\ \hline
                \textsf{TB} & $\textbf{9.039\%}$  \\                     \hline
                \textsf{WT} & $3.617\%$ \\          \hline
                \textsf{SB} & $2.618\%$ \\
                \hline
        \end{tabular}

        \label{table2}
\end{table}

Here, we want to confirm that the discrete-time regression models are good alternatives to the Cox regression model. We compare the results obtained from the non-private discrete-time regression models without regularization term to the results obtained from Cox regression. We use the relative error (RE) which is defined as:
$$
RE = \frac {\|\beta - \beta^*\|} {\|\beta^*\|}
$$ where $\beta$ is from the discrete-time regression with logit link and $\beta^*$ is from Cox regression. The results are shown in Table~\ref{table2}. We observe that the results obtained from the discrete-time regressions are very similar to the results obtained from the Cox regression with relative errors ranging from $2\%-9\%$. At the worse case of the data set \TB, the parameter obtained from the discrete-time model $\beta = [0.0122443, -0.849823, -0.239539]\p$ is still a good approximation of the parameter obtained from the Cox model $\beta^* = [0.0585478, -0.790977, -0.23906]\p$. As such, these results confirm that the discrete-time regression models are good alternatives to the Cox regression in practice.

\section{Conclusion}
\label{sec7}
In this work, we propose solutions for the problem of protecting differential privacy for discrete-time regression models used in survival analysis. In particular, we extend the perturbation approaches to a generalized form in which the loss function is a sum of logistic loss functions. In addition, we propose a sampling approach to practically protect differential privacy by sampling a scaled posterior distribution with the \textsf{pSGLD} sampling algorithm.
Even though we focus our work on discrete-time survival regression, our proposed approaches can be applied to other problems with similar loss functions as well. Moreover, our proposed approaches can be easily extended to discrete-time regression models in which the explanatory variables are changed over time.
For further work, a differentially private version of Cox regression would be a good complement to our work.
 
\begin{acks}
  The authors would like to thank Dr. Xiaokui Xiao for his insightful comments on the privacy problem of Cox regression.
\end{acks}

\bibliographystyle{ACM-Reference-Format}
\bibliography{main}

\end{document}